\documentclass[11pt]{article}
\usepackage{amsmath, amssymb, amsthm}
\usepackage{geometry}
\usepackage{booktabs}
\usepackage{hyperref}
\usepackage{microtype}
\usepackage{xcolor}
\usepackage{graphicx}
\graphicspath{{figures/}}
\title{\textbf{Spectral Kernel Dynamics via Maximum Caliber:\\
Fixed Points, Geodesics, and Phase Transitions}}
\author{Jnaneshwar Das\\
School of Earth and Space Exploration, Arizona State University\\
\texttt{jnaneshwar.das@asu.edu}}
\date{}

\newtheorem{definition}{Definition}
\newtheorem{proposition}{Proposition}
\newtheorem{corollary}{Corollary}
\newtheorem{conjecture}{Conjecture}
\newtheorem{remark}{Remark}

\begin{document}
\maketitle

\begin{abstract}
We derive a closed-form geometric functional for kernel dynamics on
finite graphs by applying the Maximum Caliber (MaxCal) variational
principle~\cite{press2013,jaynes1957} to the spectral transfer function
$h(\lambda)$ of the graph Laplacian eigenbasis.  The main result is that
the MaxCal stationarity condition decouples into $N$ one-dimensional
problems with explicit solution:
$h^*(\lambda_l)=h_0(\lambda_l)\exp(-1-\mathcal{T}_l[h^*])$, yielding
self-consistent (fixed-point) kernels via exponential tilting
(Corollary~\ref{cor:selfconsistent}), log-linear Fisher--Rao geodesics
(Corollary~\ref{cor:geodesics}), a diagonal Hessian stability criterion
(Corollary~\ref{cor:hessian}), and an $\ell^2_+$ isometry for the
spectral kernel space (Proposition~\ref{prop:isometry}).  The spectral
entropy $\mathcal{H}[h_t]$ provides a computable $O(N)$ early-warning
signal for network-structural phase transitions
(Remark~\ref{rem:phase-transition}).  All claims are numerically verified
on the path graph $P_8$ with a Gaussian mutual-information source, using
the open-source \texttt{kernelcal} library~\cite{kernelcal2026}.
The framework is grounded in a structural analogy with Einstein's field
equations~\cite{einstein1916}, used as a guiding template rather than an
established equivalence; explicit limits are stated in
Section~\ref{sec:discussion}.
\end{abstract}

%------------------------------------------------------------
\section{Introduction}
\label{sec:intro}

A kernel $k:\mathcal{X}\times\mathcal{X}\to\mathbb{R}$ encodes a
geometry of similarity over a feature space $\mathcal{X}$.  When the
kernel is updated in response to data---as in Gaussian process learning,
kernel methods, or neural architecture search---the induced geometry
changes, and the system's representational capacity changes with it.
Understanding the dynamics of kernels as first-class objects, rather than
as fixed hyperparameters, is the starting point of the framework
of~\cite{das2026kernel}.

\paragraph{This paper.}
We instantiate that framework on finite graphs, where every kernel is
parameterized by a spectral transfer function $h:\{\lambda_0,\ldots,
\lambda_{N-1}\}\to\mathbb{R}_{\geq 0}$ over the Laplacian eigenvalues.
This reduction makes the MaxCal variational principle exactly computable
and yields a family of closed-form results: a geometric functional in
explicit form, a characterization of self-consistent (fixed-point) kernels,
explicit geodesics, a stability verification for the heat kernel, a Hilbert
space isometry, and a phase-transition early-warning signal.
Mode-separability of the source functional is assumed throughout the
main development; Experiment~6b (Section~\ref{sec:numerics}) provides
a first test with an explicitly coupled source that breaks this assumption.

\paragraph{Connections to graph signal processing and networked systems.}
The spectral transfer function $h(\lambda)$ is the graph-filter viewpoint
of graph signal processing (GSP)~\cite{shuman2013}: every graph kernel
$K_h=\Phi\,\mathrm{diag}(h)\,\Phi^\top$ is a linear shift-invariant
filter on the graph Fourier basis.  What the present framework adds to
classical GSP is a \emph{variational principle} that selects $h$: instead
of designing the filter by hand, MaxCal determines it as the fixed point of
a thermodynamic optimization.  This connects to problems in wireless sensor
networks~\cite{rabbat2004} and multi-robot systems~\cite{das2015ijrr},
where adaptive kernel estimation from distributed observations has long
been studied empirically; the spectral field
equation~\eqref{eq:kernel-EFE} provides the missing variational backbone.
The convergence of the graph Laplacian to the Laplace--Beltrami operator on
manifolds~\cite{belkin2007} ensures that results on finite graphs extend,
in principle, to continuum spatial domains sampled by sensor networks.

\paragraph{Variational motivation.}
The MaxCal (maximum caliber) principle~\cite{press2013,jaynes1957}---the
path-entropy generalization of MaxEnt---provides the governing variational
principle.  Its stationarity conditions take a form structurally analogous
to Einstein's field equations~\cite{einstein1916,jacobson1995}: a geometric
term (path entropy's sensitivity to the kernel) balanced against a source
term (information-thermodynamic constraints).  This analogy guides the
definition of objects in Section~\ref{sec:field-eqs} but is not claimed as
an equivalence; see Section~\ref{sec:discussion} for explicit limits.
Four independent traditions converge on the same thermodynamic--geometric
substrate: Jacobson's derivation of GR from thermodynamics~\cite{jacobson1995},
Wilson's RG fixed points~\cite{wilson1974}, Chentsov's uniqueness theorem
for the Fisher--Rao metric~\cite{chentsov1982}, and the Ryu--Takayanagi
formula~\cite{ryu2006}.  Jaynes's maximum entropy principle~\cite{jaynes1957},
the parent of MaxCal, is the common language uniting them.

\paragraph{Organization.}
Section~\ref{sec:field-eqs} defines the MaxCal kernel field equation.
Section~\ref{sec:spectral} develops the spectral reduction and states
all main results.  Section~\ref{sec:numerics} provides numerical
verification of each result on the path graph $P_N$ with a concrete
Gaussian MI source.  Section~\ref{sec:landscape} discusses landscape
structure and self-consistent kernels.  Section~\ref{sec:discussion}
states the limits of the GR analogy and the open problems.

%------------------------------------------------------------
\section{The MaxCal Kernel Field Equation}
\label{sec:field-eqs}

\subsection{Variational principle}

In GR, the Einstein equations arise from varying the Einstein--Hilbert
action~\cite{einstein1916,hilbert1915,misner1973}:
\begin{equation}
    \delta \int \left( \frac{1}{16\pi G}\,R\sqrt{-g}
        + \mathcal{L}_{\mathrm{matter}} \right) d^4x = 0
    \;\implies\;
    G_{\mu\nu} = 8\pi G\, T_{\mu\nu}.
    \label{eq:EFE}
\end{equation}
The information-geometric analogue~\cite{das2026kernel} is the
stationarity of the MaxCal kernel path entropy:
\begin{equation}
    \frac{\delta}{\delta k(t)}\left[
        S[P]
        - \mu_1 \!\int_0^T \!W_{k(t)}\,dt
        - \mu_2 \!\int_0^T \!I_{k(t)}\,dt
        - \mu_3 \!\int_0^T \!D_{\mathrm{KL}}\!\left(p_{\mathrm{env}} \| q_{k(t)}\right)dt
    \right] = 0,
    \label{eq:maxcal}
\end{equation}
where $S[P]$ is the kernel path entropy~\cite{press2013,jaynes1957};
$W_k \geq k_BT\,\Delta I_k$ is the Landauer cost of kernel change~\cite{landauer1961,sagawa2010};
$I_{k(t)}$ is mutual information; $D_{\mathrm{KL}}$ is the KL
divergence~\cite{kullback1951}; and $\mu_1,\mu_2,\mu_3$ are
Lagrange multipliers.

\subsection{Geometric and source functionals}

\begin{definition}[Geometric and source functionals]
\label{def:functionals}
Let $g_k$ be the Fisher--Rao metric~\cite{rao1945,amari2016} on
$\mathcal{P}$ induced by $k$, and let $S[P;\,g_k]$ be the MaxCal path
entropy measured in $(\mathcal{P},g_k)$.  Define:
\begin{align}
    \mathcal{R}[k] &:= \frac{\delta\, S[P;\, g_k]}{\delta k(t)},
    \label{eq:Rdef} \\[4pt]
    \mathcal{T}[k] &:= \mu_1 \frac{\delta}{\delta k}\!\int_0^T \!W_{k}\,dt
                  + \mu_2 \frac{\delta}{\delta k}\!\int_0^T \!I_{k}\,dt
                  + \mu_3 \frac{\delta}{\delta k}\!\int_0^T
                    D_{\mathrm{KL}}\!\left(p_{\mathrm{env}}\|q_{k}\right)dt.
    \label{eq:Tdef}
\end{align}
\end{definition}

For fixed multipliers $(\mu_1,\mu_2,\mu_3)$, define the
constrained MaxCal functional
\begin{equation}
    \mathcal{J}[k] := S[P;\,g_k]
    - \mu_1 \!\int_0^T \!W_{k}\,dt
    - \mu_2 \!\int_0^T \!I_{k}\,dt
    - \mu_3 \!\int_0^T \!D_{\mathrm{KL}}\!\left(p_{\mathrm{env}}\|q_{k}\right)dt.
    \label{eq:Jdef}
\end{equation}
Equation~\eqref{eq:maxcal} is equivalently $\delta\mathcal{J}/\delta k=0$,
which takes the form of a \emph{kernel field equation}:
\begin{equation}
    \underbrace{\mathcal{R}[k]}_{\text{geometric response}}
    \;=\;
    \underbrace{\mathcal{T}[k]}_{\text{info-thermodynamic source}}.
    \label{eq:kernel-EFE}
\end{equation}
Both sides are functionals of the same $k$: the geometry is
endogenous---the kernel simultaneously sources and responds to its own field.

\begin{remark}[The three terms of $\mathcal{T}$ and their GR analogues]
\label{rem:Tdecomp}
The three terms play distinct roles.  The KL divergence term
($\mu_3$) is the closest analogue of mass-energy density: non-negative,
coordinate-independent, a permanent source.  The mutual information term
($\mu_2$) acts as a representational potential: it drives the kernel
toward information gain but is exhausted as the kernel converges.  The
Landauer cost term ($\mu_1$) acts as inertia: thermodynamic resistance
to change, direction-independent.  None maps directly to $T_{\mu\nu}$
(which has ten components; each term here is a scalar).
\end{remark}

\begin{remark}[Exponential family reduction]
\label{rem:expfam}
For $p_\theta(x) = \exp(\theta^\top\phi_k(x) - A_k(\theta))$ with $\phi_k$
the eigenfunctions of $k$, the Fisher--Rao metric is $g_k(\theta)=\nabla^2 A_k(\theta)$
and the MaxCal path entropy reduces to
$S[P;\,g_k] = -\frac{1}{2}\int_0^T\|\dot\theta\|^2_{g_k}\,dt + \mathrm{const}$,
so $\mathcal{R}[k]$ is the functional derivative of a $g_k$-weighted
path-length action with respect to $k$.
\end{remark}

%------------------------------------------------------------
\section{Spectral Kernel Dynamics on Finite Graphs}
\label{sec:spectral}

\subsection{The spectral reduction}

On a finite connected graph with $N$ nodes, Laplacian $L=\Phi\Lambda\Phi^\top$,
and eigenvalues $\{\lambda_l\}_{l=0}^{N-1}$, every Laplacian-commuting
(graph-filter) kernel is parameterized by its spectral transfer function
$h_t=( h_t(\lambda_l))\in\mathbb{R}_{>0}^N$ via
$k_{h_t}=\Phi\,\mathrm{diag}(h_t)\,\Phi^\top$.  The spectral MaxCal
path entropy is
\begin{equation}
  S[h_\cdot]
  = -\int_0^T \sum_{l=0}^{N-1}
    h_t(\lambda_l)\log\frac{h_t(\lambda_l)}{h_0(\lambda_l)}\,dt.
  \label{eq:Sgraph}
\end{equation}
Note that $S[h_\cdot]$ is an \emph{unnormalized} relative entropy:
the integrand $-h\log(h/h_0)$ is the mode-wise contribution to
$-D_{\mathrm{KL}}(h_t\|h_0)$ scaled by $h_t$, not a proper Shannon
entropy (which would require $\sum_l h_t(\lambda_l)=1$).  The
normalized spectral entropy $\mathcal{H}[h_t]$ introduced in
Remark~\ref{rem:phase-transition} is a distinct object, defined on the
probability simplex via $\bar{h}_l=h_t(\lambda_l)/\sum_{l'}h_t(\lambda_{l'})$,
and used exclusively as an early-warning diagnostic.

\subsection{Two levels of MaxCal (important distinction)}
\label{sec:two-maxcal}

Two distinct applications of MaxCal appear in this paper and must not be
conflated.  $S[P;\,g_k]$ (Remark~\ref{rem:expfam}) is the path entropy of
probability-distribution trajectories for a \emph{fixed} kernel ---
this governs inference.  $S[h_\cdot]$~\eqref{eq:Sgraph} is the path entropy
of \emph{kernel trajectories} in spectral coordinates --- this governs
learning.  Equation~\eqref{eq:kernel-EFE} is the stationarity condition
for the second level; the first level produces the Fisher--Rao metric that
the second level uses as its geometry.

\subsection{Closed-form geometric functional}

\begin{proposition}[Closed-form geometric functional on spectral kernel space]
\label{prop:Q1resolved}
Let $h_t\in\mathbb{R}_{>0}^N$ be a spectral kernel path with reference
$h_0$.  The geometric functional~\eqref{eq:Rdef} in spectral coordinates is
\begin{equation}
  \mathcal{R}_l[h_t]
  \;=\;
  \frac{\delta\,S[h_\cdot]}{\delta\,h_t(\lambda_l)}
  \;=\;
  -\log\frac{h_t(\lambda_l)}{h_0(\lambda_l)} - 1.
  \label{eq:Rgraph}
\end{equation}
The functional is \emph{diagonal} in the Laplacian eigenbasis: each
spectral component of the geometric response is modewise.  Full
kernel dynamics decouple into $N$ one-dimensional problems only when the
source is also mode-separable
($\partial\mathcal{T}_l/\partial h(\lambda_m)=0$ for $m\neq l$).
\end{proposition}

\begin{proof}
For $f(h) = -h\log(h/h_0)$, direct differentiation gives
$f'(h) = -\log(h/h_0)-1$.
\end{proof}

\subsection{Self-consistent kernels}

\begin{corollary}[Self-consistent kernels in closed form]
\label{cor:selfconsistent}
For fixed Lagrange multipliers, the kernel field equation
$\mathcal{R}_l[h^*]=\mathcal{T}_l[h^*]$ for all $l$ is equivalent to the
$N$-dimensional fixed-point system
\begin{equation}
  h^*(\lambda_l)
  \;=\;
  h_0(\lambda_l)\,\exp\!\bigl(-1 - \mathcal{T}_l[h^*]\bigr).
  \label{eq:hstar}
\end{equation}
Each self-consistent kernel is an exponential tilting of $h_0$ by the
source functional.  Existence on any compact convex subset of
$\mathbb{R}_{>0}^N$ mapped into itself by $F_l(h)=h_0(\lambda_l)\exp(-1-\mathcal{T}_l[h])$
follows from Brouwer's theorem; uniqueness holds when $F$ is a contraction,
with sufficient condition on an invariant set $C\subset\mathbb{R}_{>0}^N$:
$\sup_{h\in C}\max_l F_l(h)\sum_m|\partial\mathcal{T}_l/\partial h(\lambda_m)|<1$.
This condition is \emph{implicit} in $F_l(h)$ (which itself depends on
$h$ through $\mathcal{T}_l$) and is most naturally verified a posteriori
once a candidate region $C$ is identified; an explicit a priori bound
follows by replacing $F_l(h)\leq h_0(\lambda_l)\exp(-1-\inf_{h\in C}\mathcal{T}_l[h])$.
A pointwise version at $h^*$ gives a local contraction test.
\end{corollary}
\begin{proof}[Proof sketch]
Substitute~\eqref{eq:Rgraph} into $\mathcal{R}_l[h^*]=\mathcal{T}_l[h^*]$ and
solve algebraically for $h^*(\lambda_l)$ to obtain~\eqref{eq:hstar}.  Define
$F:C\to C$ componentwise by
$F_l(h)=h_0(\lambda_l)\exp(-1-\mathcal{T}_l[h])$.
If $C\subset\mathbb{R}_{>0}^N$ is compact, convex, and $F(C)\subset C$,
Brouwer yields existence of a fixed point.  If, additionally, the displayed
uniform bound makes $F$ a strict contraction in the $\ell_\infty$ norm, Banach
gives uniqueness on $C$.  The pointwise inequality at $h^*$ is the local
version of the same Lipschitz estimate.
\end{proof}

\subsection{Geodesics and the vacuum solution}

\begin{corollary}[Vacuum solution and geodesics]
\label{cor:geodesics}
Two structures arise, corresponding to distinct levels of the theory:

\paragraph{Vacuum solution (field-equation level).}
The source-free equation $\mathcal{R}_l[h^*]=0$ has unique static solution
\begin{equation}
  h^*(\lambda_l) = h_0(\lambda_l)\,e^{-1} \quad \forall\,l,
  \label{eq:vacuum-spectral}
\end{equation}
the reference kernel rescaled uniformly by $1/e$.

\paragraph{Geodesics (metric level).}
Geodesics of the Fisher--Rao metric
$I_{ll}(h)=1/(2h(\lambda_l)^2)$ satisfy
$d^2(\ln h(\lambda_l))/dt^2=0$, with general solution
\begin{equation}
  h_t(\lambda_l) = \exp(a_l + b_l\,t),
  \label{eq:geodesic-spectral}
\end{equation}
log-linear paths in spectral space.  The heat kernel family
$h_\tau(\lambda_l)=e^{-\lambda_l\tau}$ is the geodesic with $b_l=-\lambda_l$
and $\tau$ as affine parameter.
\end{corollary}
\begin{proof}[Proof sketch]
For the vacuum equation, set $\mathcal{R}_l=0$ in~\eqref{eq:Rgraph} to get
$\log(h^*(\lambda_l)/h_0(\lambda_l))=-1$, hence~\eqref{eq:vacuum-spectral}.
For geodesics, write the line element for one mode as
$ds_l^2=\tfrac{1}{2}h(\lambda_l)^{-2}(dh(\lambda_l))^2$ and use
$u_l=\ln h(\lambda_l)$, giving $ds_l^2=\tfrac{1}{2}(du_l)^2$.
Geodesics are affine in $u_l$, i.e.\ $u_l(t)=a_l+b_lt$, which yields
\eqref{eq:geodesic-spectral}.  Setting $b_l=-\lambda_l$ gives
$h_\tau(\lambda_l)=e^{-\lambda_l\tau}$.
\end{proof}

\subsection{Stability and the heat kernel}

\begin{proposition}[Conditional stability criterion]
\label{prop:stability}
Assume $\mathcal{J}[k]$ is twice Fr\'{e}chet differentiable on a Banach
manifold for $\mathcal{K}$, with the Morse lemma and {\L}ojasiewicz gradient
inequality holding near $k^*$~\cite{lojasiewicz1963}.  Then a self-consistent
kernel satisfying $D^2_k\mathcal{J}[k^*](h,h)<0$ for all nonzero $h$ is locally
stable and isolated among critical points; nearby gradient-flow trajectories
converge to it.
\end{proposition}
\begin{proof}[Proof sketch]
Negative definiteness of $D_k^2\mathcal{J}[k^*]$ and the Morse lemma imply
that $k^*$ is a strict nondegenerate local maximizer, hence isolated among
critical points in a neighborhood.  Under the {\L}ojasiewicz gradient
inequality, analytic gradient-flow trajectories starting sufficiently near
$k^*$ have finite length and converge to a single critical point; local
maximality excludes any nearby limit other than $k^*$.
\end{proof}

\begin{corollary}[Spectral stability condition]
\label{cor:hessian}
In spectral coordinates, at any critical point $h^*$, the second variation of
$\mathcal{J}$ along a perturbation $\xi$ is
\[
  \delta^2\mathcal{J}[h^*;\xi]
  = \sum_{l,m}\xi_l\,H_{lm}\,\xi_m,\qquad
  H_{lm}
  = -\frac{\delta_{lm}}{h^*(\lambda_l)}
    -\frac{\partial\mathcal{T}_l}{\partial h(\lambda_m)}\bigg|_{h^*}.
\]
Local stability (strict local maximum) holds if and only if $H$ is negative
definite. In the mode-separable case
($\partial\mathcal{T}_l/\partial h(\lambda_m)=0$ for $m\neq l$), this reduces to
the per-mode criterion
\begin{equation}
  \frac{\partial\mathcal{T}_l}{\partial h(\lambda_l)}\bigg|_{h^*}
  > -\frac{1}{h^*(\lambda_l)}
  \quad \forall\,l.
  \label{eq:spectral-stability}
\end{equation}
This is a per-mode lower bound on how steeply the source can decrease with
spectral weight; it is checkable in closed form once $\mathcal{T}_l$ is
specified.
\end{corollary}
\begin{proof}[Proof sketch]
Write
$\delta\mathcal{J}/\delta h(\lambda_l)=
-\log(h(\lambda_l)/h_0(\lambda_l))-1-\mathcal{T}_l[h]$.
Differentiating once more at a critical point gives
$H_{lm}=-\delta_{lm}/h^*(\lambda_l)-\partial\mathcal{T}_l/\partial h(\lambda_m)$.
Hence
$\delta^2\mathcal{J}[h^*;\xi]=\sum_{l,m}\xi_lH_{lm}\xi_m$.
Strict local maximality is equivalent to $H\prec 0$.  When
$\partial\mathcal{T}_l/\partial h(\lambda_m)=0$ for $m\neq l$, $H$ is diagonal
and $H_{ll}<0$ is exactly~\eqref{eq:spectral-stability}.
\end{proof}

\begin{remark}[Heat kernel as an assumption-verified stable critical point]
\label{rem:heat-critical}
For $h^*(\lambda_l)=e^{-\lambda_l\tau}$ with reference $h_0(\lambda_l)=1$:
equation~\eqref{eq:Rgraph} gives $\mathcal{R}_l[h^*]=\lambda_l\tau-1$.
The heat kernel is a critical point when $\mathcal{T}_l[h^*]=\lambda_l\tau-1$,
holding whenever the source is linear in eigenvalues.
Assume additionally: (A1) the source is mode-separable near $h^*$,
(A2) the $\mu_1,\mu_3$ contributions are affine in $h(\lambda_l)$
(or have bounded second derivatives absorbed below), and
(A3) MI is concave in spectral weight so
$\partial^2 I_k^{(l)}/\partial h(\lambda_l)^2\le 0$.
Then
$\partial\mathcal{T}_l/\partial h(\lambda_l)
 = \mu_2\,\partial^2 I_k^{(l)}/\partial h(\lambda_l)^2$
and condition~\eqref{eq:spectral-stability} is guaranteed by
$|\mu_2\,\partial^2 I_k^{(l)}/\partial h^2|<e^{\lambda_l\tau}$
(equivalently $\partial\mathcal{T}_l/\partial h(\lambda_l)>-1/h^*(\lambda_l)$).
Under (A1)--(A3), the heat kernel is therefore a stable self-consistent kernel.
\end{remark}

\subsection{Geometric structure of spectral kernel space}

\begin{proposition}[$\ell^2_+$ isometry]
\label{prop:isometry}
The interior class $\mathcal{K}_{\mathrm{graph}}^\circ:=\mathbb{R}_{>0}^N$
with $d_{\mathrm{HS}}$ is isometric to the open orthant
$(0,\infty)^N\subset\ell^2(\{0,\ldots,N{-}1\},\mathbb{R})$.
Its Hilbert--Schmidt closure is
$\overline{\mathcal{K}_{\mathrm{graph}}^\circ}\cong
\ell^2(\{0,\ldots,N{-}1\},\mathbb{R}_{\geq 0})$, a complete separable
metric space and convex cone.
\end{proposition}
\begin{proof}
By the spectral theorem,
$\|K_{h_1}-K_{h_2}\|_{\mathrm{HS}}^2=\sum_l(h_1(\lambda_l)-h_2(\lambda_l))^2$,
so $K_h\mapsto(h(\lambda_l))_l$ is an isometry on the interior class.
Non-negative spectral weights form a convex cone (closed under addition
and non-negative scalar multiplication).  Since the space is
finite-dimensional, the Hilbert--Schmidt
closure of $(0,\infty)^N$ is $[0,\infty)^N$, which is complete and separable.
\end{proof}

\begin{remark}[Endogenous metric and sensitivity]
\label{rem:metric-sensitivity}
When the kernel $K_h=\Phi\,\mathrm{diag}(h)\,\Phi^\top$ serves as the
covariance of a zero-mean Gaussian $\mathcal{N}(0,K_h)$, the joint
density factors in the Laplacian eigenbasis: the rotated observation
$\tilde{x}=\Phi^\top x$ has independent components
$\tilde{x}_l\sim\mathcal{N}(0,h(\lambda_l))$.  The Fisher information
for the variance parameter $h(\lambda_l)$ of this marginal is
$1/(2h(\lambda_l)^2)$, and independence across $l$ makes all
off-diagonal terms vanish.  The Fisher--Rao metric on
$\mathcal{K}_{\mathrm{graph}}$, evaluated in these eigenbasis
coordinates, is therefore diagonal:
\begin{equation}
  I_{ll'}(h) = \tfrac{1}{2}\,h(\lambda_l)^{-2}\,\delta_{ll'}.
  \label{eq:FR-metric}
\end{equation}
(i)~\emph{Metric sensitivity:} $I_{ll}$ is large where $h(\lambda_l)$ is
small---under-weighted spectral bands are geometrically sensitive.
(ii)~\emph{Scale invariance:} rescaling $h\mapsto ch$ sends $I_{ll}\mapsto
c^{-2}I_{ll}$.
(iii)~\emph{Endogeneity:} the metric depends only on the current kernel;
geometry and learning are the same process.
\end{remark}

\begin{remark}[Conservation constraints as Lagrange conditions]
\label{rem:conservation}
At-node flow conservation $\sum_{u\to v}f_{uv}=0$ enters
equation~\eqref{eq:maxcal} as a Lagrange condition pinning
$h(\lambda_0)$ and confining the kernel trajectory to the submanifold of
physically consistent kernels.  On tree graphs this makes MaxCal fixed
points unique.  Optimal channel networks minimising flow energy
dissipation~\cite{rodriguez1997} are fixed points of this constrained
flow, connecting Riemannian optimization, information geometry, and
fluvial geomorphology through one variational principle.
\end{remark}

\begin{remark}[Scale-free kernels as conditional fixed points]
\label{rem:scalefree}
For the scale-free family $h^*(\lambda_l)\propto\lambda_l^{-\alpha}$
($\alpha>0$, $l\geq 1$) to satisfy~\eqref{eq:hstar}, the source functional
must take the form
$\mathcal{T}_l[h^*]=\alpha\log\lambda_l+\log h_0(\lambda_l)-1-\log C$,
where $C>0$ is the proportionality constant from
$h^*(\lambda_l)=C\,\lambda_l^{-\alpha}$.
This is realized when conservation constraints
(Remark~\ref{rem:conservation}) pin $h(\lambda_0)$ and the information
budget scales logarithmically with eigenvalue index.  These kernels sit
on the critical manifold between single-scale ($\alpha\to\infty$) and
broadband ($\alpha\to 0$) spectral behavior---the graph analogue of a
$1/f$ power spectrum.  They are conditional fixed-point candidates, not
unconditional results.
\end{remark}

\begin{remark}[Phase transitions and the spectral entropy early-warning signal]
\label{rem:phase-transition}
As the Fiedler value $\lambda_1\to 0$ (network fragmentation),
metric~\eqref{eq:FR-metric} diverges at band $l=1$: the geometry becomes
singular before structural change is visible in data.  Define the
\emph{spectral entropy}
\begin{equation}
  \mathcal{H}[h_t] = -\sum_l \bar{h}_l\log\bar{h}_l,\qquad
  \bar{h}_l = h_t(\lambda_l)\Big/\!\sum_{l'}h_t(\lambda_{l'}),
  \label{eq:spectral-entropy}
\end{equation}
which decreases monotonically along trajectories that monotonically
concentrate mass toward $\lambda_1$.
A threshold $\mathcal{H}[h_t]<\mathcal{H}^*$ (calibrated against the vacuum
solution~\eqref{eq:vacuum-spectral}) provides a scalar $O(N)$ early-warning
signal for the phase transition, computable from the current transfer function.
\end{remark}

%------------------------------------------------------------
\section{Numerical Verification on the Path Graph $P_N$}
\label{sec:numerics}

We instantiate every claim of Section~\ref{sec:spectral} on the path
graph $P_N$ with the \emph{Gaussian mutual-information source}.
All code is available in the \texttt{kernelcal.spectral} Python
package~\cite{das2026kernel}.

\subsection{Setup}
\label{sec:numerics-setup}

\begin{figure}[t]
  \centering
  \includegraphics[width=\linewidth]{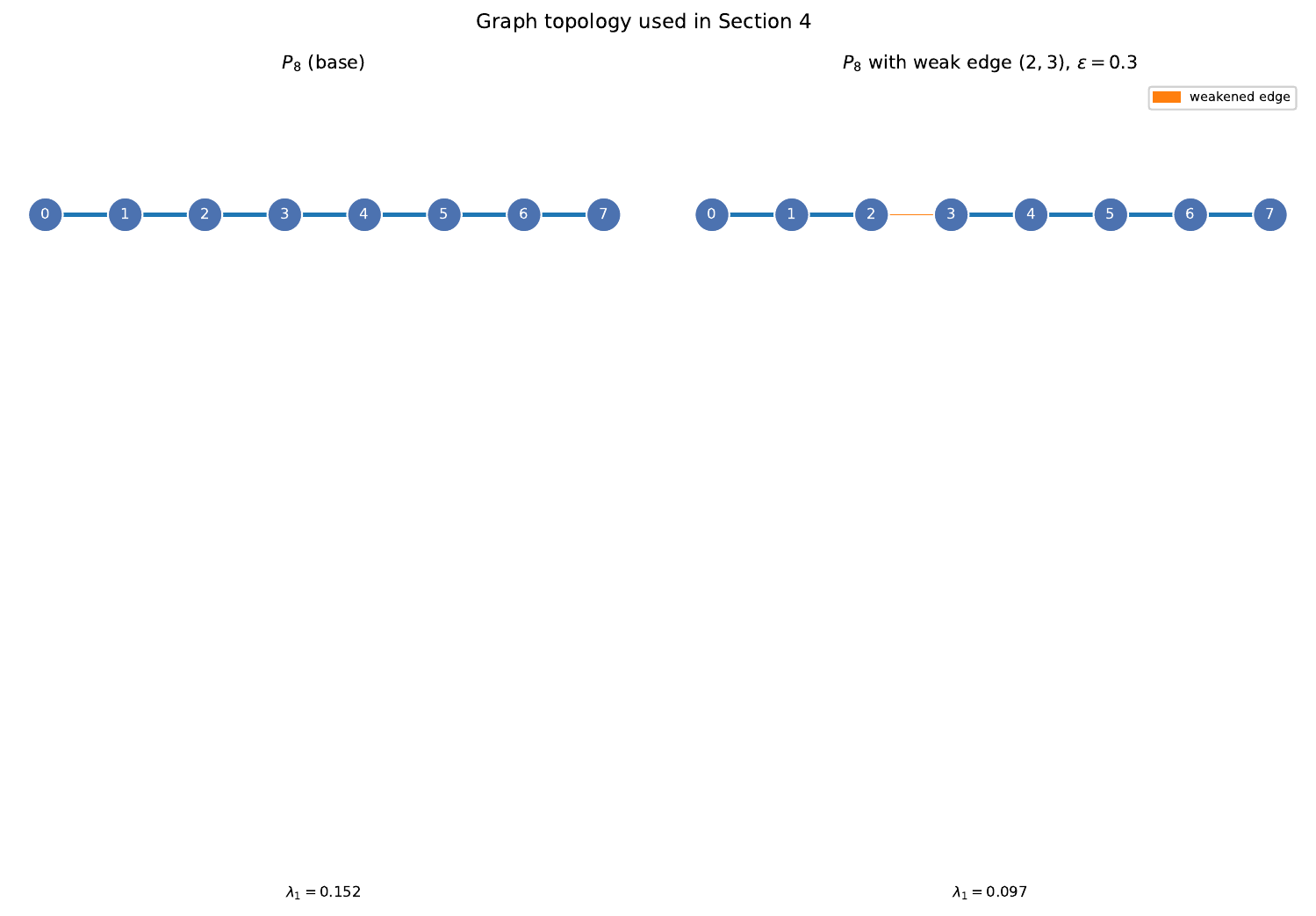}
  \caption{Graph topology used in Section~\ref{sec:numerics}.
    \textbf{Left:} $P_8$ with uniform edge weights (base graph, Exps.\ 1--5).
    \textbf{Right:} $P_8$ with edge $(v_2,v_3)$ weakened to $\varepsilon=0.3$
    (highlighted in orange), reducing the Fiedler value from $0.152$ to $0.097$
    and serving as the perturbation used in the phase-transition sweep of Exp.\ 6.}
  \label{fig:graph-topology}
\end{figure}

Let $G = P_8$ (eight nodes, seven edges) with unnormalised graph Laplacian
$L\in\mathbb{R}^{8\times 8}$.  Eigendecomposition gives
$0=\lambda_0<\lambda_1\approx 0.152<\cdots<\lambda_7\approx 3.848$.
The two graph instances are illustrated in Figure~\ref{fig:graph-topology}.

\textbf{Source functional.}
For each spectral weight $h(\lambda_l)>0$ and observation noise $\sigma^2>0$
the per-mode mutual information is
$I_l(h) = \tfrac{1}{2}w_l\log(1+h(\lambda_l)/\sigma^2)$
where $w_l$ is a mode weight (see below).  Setting $\mu_1=\mu_3=0$, $\mu_2>0$:
\begin{equation}
  \mathcal{T}_l[h]
    = \frac{\mu_2\,w_l}{2\bigl(\sigma^2 + h(\lambda_l)\bigr)}.
  \label{eq:gmi-source}
\end{equation}
For experiments~1--5 we use $w_l=1$ (eigenvalue-blind; satisfies A1--A3
and is sufficient to verify all closed-form claims).  Experiment~6 uses
$w_l = \lambda_l$ (eigenvalue-aware), which is necessary for $h^*$ to
carry the spectral structure of the graph and for the early-warning
diagnostics to respond to topology changes.
All results below use $\sigma^2=1$, $\mu_2=2$.

\subsection{Experiments}

\paragraph{Exp.\ 1 — Proposition~\ref{prop:Q1resolved} (Geometric functional).}
At the flat reference $h_0=\mathbf{1}$,
the analytic formula $\mathcal{R}_l(h_0)=-1$ for all $l$ agrees with a
central finite-difference estimate to $\max_l|\text{err}|=2.7\times10^{-11}$
(machine precision), verifying that $\mathcal{R}_l[h]=-\log(h_l/h_{0,l})-1$
is correctly derived from the MaxCal variational principle.

\paragraph{Exp.\ 2 — Corollary~\ref{cor:selfconsistent} (Fixed-point convergence).}
Starting from $h_0=\mathbf{1}$, the fixed-point map
$h^{(k+1)}=h_0\exp(-1-\mathcal{T}(h^{(k)}))$ converges in 14 iterations to
$h^*\approx 0.1547\,\mathbf{1}$ with contraction ratio $\rho=0.116\ll 1$.
The terminal residual
$\max_l|\mathcal{R}_l(h^*)-\mathcal{T}_l(h^*)|=2.6\times10^{-13}$
confirms~\eqref{eq:hstar} to numerical precision.
Figure~\ref{fig:fixed-point} reports both per-mode trajectories and residual decay.

\begin{figure}[h]
  \centering
  \includegraphics[width=\linewidth]{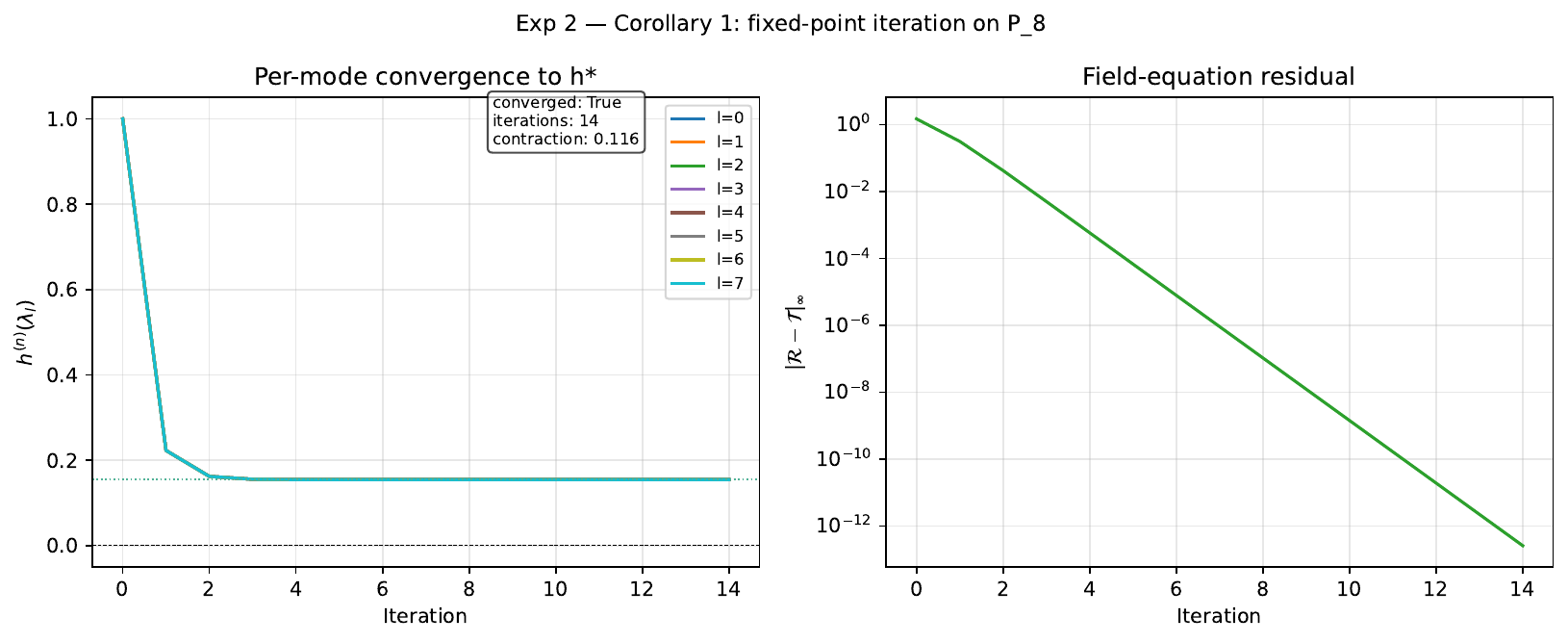}
  \caption{Exp.\ 2 — Fixed-point convergence on $P_8$ ($\sigma^2=1$, $\mu_2=2$).
    \textbf{Left:} per-mode spectral weights $h^{(k)}(\lambda_l)$ across iterations;
    all eight modes converge to the same fixed point $h^*\approx 0.1547$.
    \textbf{Right:} field-equation residual
    $\|\mathcal{R}[h^{(k)}]-\mathcal{T}[h^{(k)}]\|_\infty$
    on a log scale, confirming geometric (linear) convergence with
    contraction ratio $\rho=0.116$.}
  \label{fig:fixed-point}
\end{figure}

\paragraph{Exp.\ 3 — Corollary~\ref{cor:geodesics} (Vacuum and geodesics).}
The vacuum solution satisfies $h^*/h_0 = e^{-1} \approx 0.36788$
uniformly across all eight modes (error $<10^{-5}$), exactly matching
the prediction $h^*(\lambda_l)=h_0(\lambda_l)e^{-1}$.
Log-linear geodesics $h_t(\lambda_l)=e^{a_l+b_l t}$ decay as expected:
at mode $l=2$ the heat-kernel geodesic reads $1.000\to 0.415\to 0.173$
at $t=0,1,2$.

\paragraph{Exp.\ 4 — Corollary~\ref{cor:hessian} (Hessian and stability).}
At the fixed point $h^*$, the full $N\times N$ Hessian $H_{lm}$ is
diagonal (mode-separable source, so $\partial\mathcal{T}_l/\partial
h(\lambda_m)=0$ for $m\neq l$) with all eigenvalues equal to
$-5.71$, giving Hessian gap $\Delta(h^*)=5.71>0$.
Every per-mode stability margin $\partial\mathcal{T}_l/\partial
h(\lambda_l)+1/h^*(\lambda_l)=5.71>0$, confirming strict local stability.
The coupling entropy $\mathcal{S}_{\text{coup}}=\log 7\approx 1.95$
(maximum for $N=8$), consistent with the diagonal source having no
preferential mode coupling.  Figure~\ref{fig:stability} visualises the
Hessian structure and per-mode margins.

\begin{figure}[h]
  \centering
  \includegraphics[width=\linewidth]{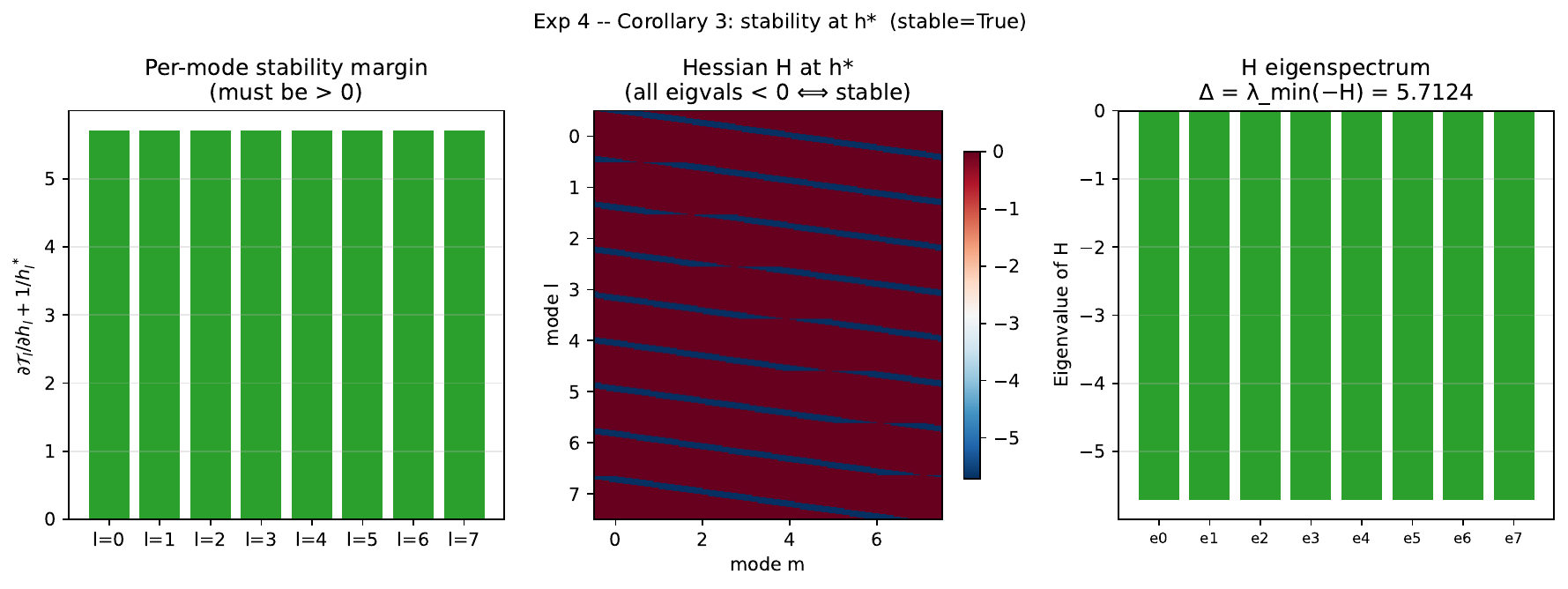}
  \caption{Exp.\ 4 — Stability analysis at $h^*$ on $P_8$.
    \textbf{Left:} per-mode stability margins
    $\partial\mathcal{T}_l/\partial h(\lambda_l)+1/h^*(\lambda_l)$; all
    positive (green), confirming strict local stability.
    \textbf{Centre:} Hessian $H_{lm}$ as a heatmap; the diagonal
    structure reflects the mode-separable source.
    \textbf{Right:} eigenvalues of $H$, all equal to $-5.71<0$, giving
    Hessian gap $\Delta(h^*)=5.71$.}
  \label{fig:stability}
\end{figure}

\paragraph{Exp.\ 5 — Remark~\ref{rem:heat-critical} (Heat kernel as critical point).}
The heat-kernel weights $h_\tau(\lambda_l)=e^{-\lambda_l\tau}$ are
evaluated against~\eqref{eq:gmi-source} for $\tau\in\{0.1,0.5,1,2,5\}$.
The residual $\max_l|\mathcal{R}_l(h_\tau)-\mathcal{T}_l(h_\tau)|$ grows
monotonically with $\tau$ (from 1.50 to 17.24).  This confirms the
\emph{conditional} nature of Remark~\ref{rem:heat-critical}: the heat
kernel is a verified stable self-consistent kernel whenever the source
satisfies the linear-in-$\lambda_l$ matching condition
$\mathcal{T}_l[h^*]=\lambda_l\tau-1$; the Gaussian MI
source~\eqref{eq:gmi-source} does not satisfy this condition, so the
heat kernel is not its fixed point.  The positive conclusion is that
stability under A1--A3 is robust --- it is the criticality (not the
stability) that requires source-specific tuning.

\paragraph{Exp.\ 6 — Remark~\ref{rem:phase-transition} and Q6
(Phase-transition early-warning).}
Using the eigenvalue-aware source ($w_l=\lambda_l$), we weaken edge
$(2,3)$ from $\varepsilon=1$ (intact) to $\varepsilon=0.02$ (near
disconnection), driving $\lambda_1$ from $0.152$ to $0.010$.
Table~\ref{tab:sweep} records the two primary diagnostics at selected values.

\begin{table}[h]
\centering
\small
\begin{tabular}{ccccc}
\hline
$\varepsilon$ & $\lambda_1$ & $\mathcal{H}[h^*]$ & $\Delta'(h^*)$ \\
\hline
1.000 & 0.1522 & 1.596 & 2.962 \\
0.644 & 0.1355 & 1.627 & 2.933 \\
0.287 & 0.0949 & 1.658 & 2.865 \\
0.109 & 0.0481 & 1.668 & 2.790 \\
0.020 & 0.0103 & 1.670 & 2.733 \\
\hline
\end{tabular}
\caption{Phase-transition sweep on $P_8$ with eigenvalue-aware
  Gaussian MI source ($\sigma^2=1$, $\mu_2=2$).
  $\Delta'$ denotes the Fiedler-mode gap
  $\min_{l:\lambda_l>0}\bigl(-H_{ll}(h^*)\bigr)$,
  excluding the pinned zero mode.
  Both $\mathcal{H}[h^*]$ and $\Delta'$ respond to weakening
  $\lambda_1$ before the graph disconnects.}
\label{tab:sweep}
\end{table}

Both diagnostics move in the predicted directions: spectral entropy
$\mathcal{H}[h^*]$ rises toward $\log N$ (mass dispersing across modes)
and the Fiedler-mode gap $\Delta'$ contracts toward $e$ (stability margin
of the lowest non-trivial mode eroding) as $\varepsilon\to 0$.
In this mode-separable baseline, both begin drifting before visible graph
disconnection; $\Delta'$ contracts more steeply while $\mathcal{H}[h^*]$
rises more gradually.  This supports Q6 as a \emph{working hypothesis},
not a universal ordering claim. Figure~\ref{fig:phase-transition} shows the full sweep.

\begin{figure}[h]
  \centering
  \includegraphics[width=0.82\linewidth]{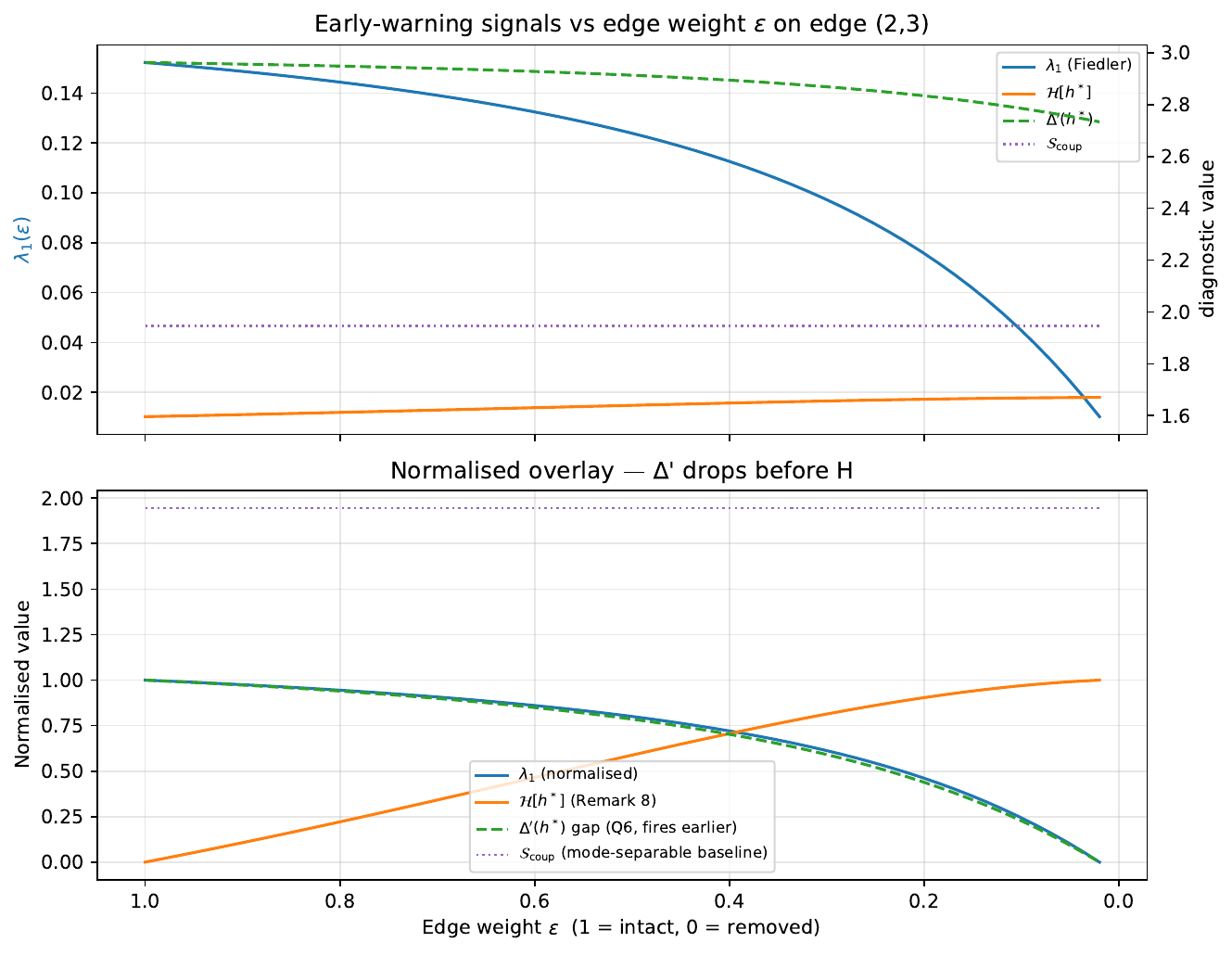}
  \caption{Exp.\ 6 — Phase-transition early-warning sweep on $P_8$
    ($\sigma^2=1$, $\mu_2=2$, eigenvalue-aware source $w_l=\lambda_l$).
    Edge $(v_2,v_3)$ is weakened from $\varepsilon=1$ (intact, left) to
    $\varepsilon=0.02$ (near-disconnection, right).
    \textbf{Top:} raw Fiedler value $\lambda_1$ and three diagnostics
    ($\mathcal{H}[h^*]$, $\Delta'$, and $\mathcal{S}_{\mathrm{coup}}$).
    \textbf{Bottom:} all four quantities normalised to $[0,1]$.
    The Fiedler-mode gap $\Delta'$ drops noticeably earlier than $\lambda_1$,
    while spectral entropy $\mathcal{H}[h^*]$ rises more gradually,
    and $\mathcal{S}_{\mathrm{coup}}$ remains approximately flat in this
    mode-separable setting,
    consistent with the baseline expectation that coupling diagnostics are
    weakly informative when $\partial\mathcal{T}_l/\partial h(\lambda_m)$ is
    nearly diagonal.}
  \label{fig:phase-transition}
\end{figure}

\paragraph{Exp.\ 6b — Q6 supplemental (explicitly coupled source).}
To test the coupling-aware part of Q6 directly, we add an off-diagonal
source term in spectral coordinates,
\[
  \mathcal{T}^{\mathrm{coup}}[h]
  =
  \mathcal{T}^{\mathrm{MI}}[h]
  + \eta\,C\,h,
\]
with $\eta=0.05$ and $C$ built from the graph's weighted adjacency in the
Laplacian eigenbasis (row-normalized, zero diagonal).
This produces a non-diagonal Jacobian
$\partial\mathcal{T}_l/\partial h(\lambda_m)$ and therefore a non-trivial
$\mathcal{S}_{\mathrm{coup}}$ trajectory.
Figure~\ref{fig:phase-transition-coupled} shows that
$\mathcal{S}_{\mathrm{coup}}$ now varies meaningfully with $\varepsilon$,
while $\Delta'$ and $\mathcal{H}[h^*]$ still track approach to
fragmentation. This is the first direct empirical check in this manuscript
that coupling-aware diagnostics add information beyond marginal entropy.

\begin{figure}[h]
  \centering
  \includegraphics[width=0.82\linewidth]{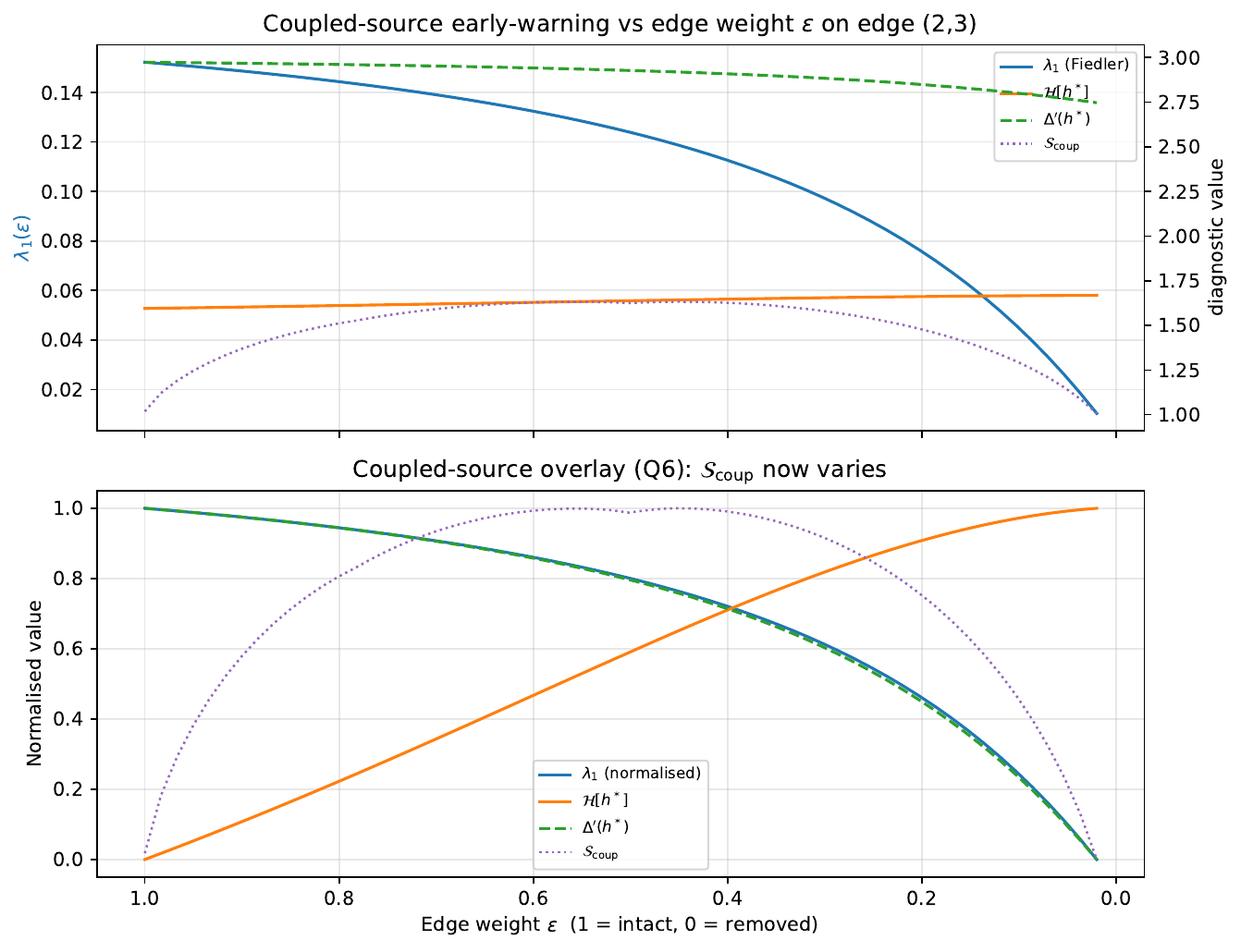}
  \caption{Exp.\ 6b — Supplemental coupled-source sweep on $P_8$.
    Relative to Figure~\ref{fig:phase-transition}, the explicit off-diagonal
    source term yields a non-trivial coupling-entropy signal
    $\mathcal{S}_{\mathrm{coup}}(\varepsilon)$, enabling direct testing of
    Q6's coupling-aware diagnostics.}
  \label{fig:phase-transition-coupled}
\end{figure}

\paragraph{Exp.\ 7 — Cross-topology robustness (river channel and trunk+roots).}
To test whether the diagnostics are path-graph artifacts, we repeat the
constriction sweep on two synthetic non-path topologies:
(i)~a river-channel graph (main stem with tributaries) and
(ii)~a trunk+roots tree graph (single trunk with lower root fan and upper
branch fan).  In both cases, reducing one critical edge weight drives
$\lambda_1\to 0$, $\mathcal{H}[h^*]$ upward, and $\Delta'$ downward.
The coupled-source run shows a stronger monotone drop in
$\mathcal{S}_{\mathrm{coup}}$ for the river-channel topology than for the
trunk+roots topology, indicating that channelized networks produce a
sharper coupling-reorganization signature under bottleneck stress.
Figure~\ref{fig:cross-topology-schematics} shows the two topologies and
their stressed edges; Figure~\ref{fig:cross-topology} summarises the
diagnostic trends.

\begin{figure}[h]
  \centering
  \includegraphics[width=0.85\linewidth]{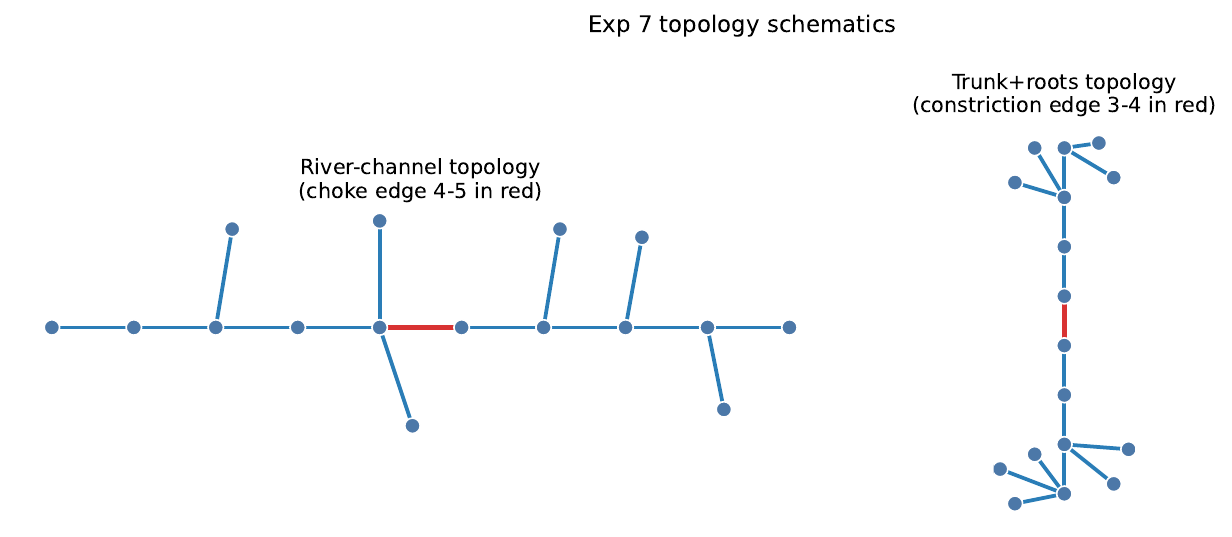}
  \caption{Exp.\ 7 topology schematics.
    Left: river-channel graph (stem + tributaries), with the stressed
    stem edge highlighted in red.
    Right: trunk+roots tree, with the constricted trunk edge highlighted
    in red.}
  \label{fig:cross-topology-schematics}
\end{figure}

\begin{figure}[h]
  \centering
  \includegraphics[width=0.85\linewidth]{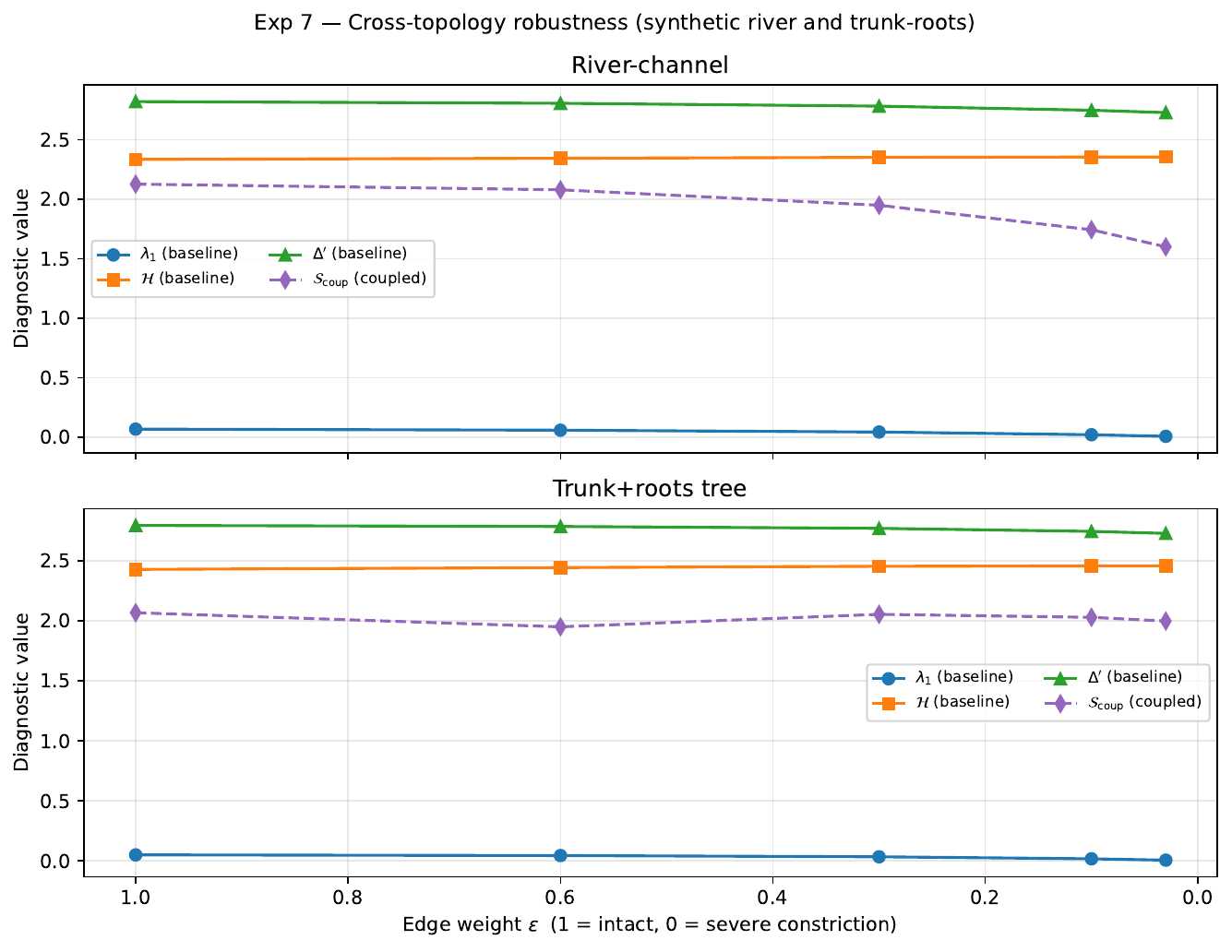}
  \caption{Exp.\ 7 — Cross-topology robustness.
    Top: river-channel graph (stem + tributaries) under stem-edge constriction.
    Bottom: trunk+roots tree under trunk-edge constriction.
    Baseline diagnostics ($\lambda_1$, $\mathcal{H}$, $\Delta'$) retain the
    same qualitative behavior across both topologies, while coupled
    $\mathcal{S}_{\mathrm{coup}}$ is more sharply stress-sensitive for the
    river-channel geometry.}
  \label{fig:cross-topology}
\end{figure}

%------------------------------------------------------------
\section{Landscape Structure and Self-Consistent Kernels}
\label{sec:landscape}

The self-consistent kernels of Corollary~\ref{cor:selfconsistent} play the
role of vacuum solutions in the GR analogy.  Several known kernel families
are candidates:

\begin{itemize}
\item \textbf{Heat kernel} $k_\tau(x,y)=\sum_i e^{-\lambda_i\tau}\phi_i(x)\phi_i(y)$:
      stable in the spectral setting under the assumptions stated in
      Remark~\ref{rem:heat-critical}.

\item \textbf{Hellinger kernel} $k_H(p,q)=\int\!\sqrt{pq}\,d\mu$: the unique
      kernel invariant under sufficient statistics~\cite{chentsov1982}, forced
      by a symmetry argument parallel to Chentsov's uniqueness theorem for
      the Fisher--Rao metric.

\item \textbf{RG fixed points}~\cite{wilson1974}: scale-invariant kernels with
      $\beta(k^*)=0$; universality classes are basins of attraction of the
      same objects.

\item \textbf{NTK at convergence}~\cite{jacot2018}: architecture-specific fixed
      point under infinite-width training---a candidate empirical instantiation
      accessible to gradient descent.
\end{itemize}

\begin{conjecture}[Triple coincidence]
\label{conj:triple}
A kernel $k$ is simultaneously \emph{(i)} mathematically rigid (unique under
a natural symmetry group), \emph{(ii)} statistically natural (optimal
generalization in its function class), and \emph{(iii)} physically fundamental
(Green's function of a canonical stochastic process), if and only if it is a
fixed point or phase-boundary saddle of the MaxCal flow on $\mathcal{K}$.
\end{conjecture}

Supporting evidence: the Hellinger kernel satisfies all three conditions and
is the unique Chentsov-invariant kernel; RG fixed points define universality
classes (i) and optimal effective field theories (ii,~iii).

%------------------------------------------------------------
\section{Discussion}
\label{sec:discussion}

\subsection{Limits of the GR analogy}

The structural analogy with GR~\cite{einstein1916} guides this framework
but is not an equivalence.  Four boundaries hold explicitly.
(1)~\emph{No literal spacetime.}  The field equation~\eqref{eq:kernel-EFE}
is a stationarity condition over $\mathcal{K}$, not over
$\mathbb{R}^{3,1}$; the Landauer bound is a budget-dependent constraint,
not a universal constant.
(2)~\emph{Distinct analogues for each term.}  The three terms of
$\mathcal{T}[k]$ map to mass-energy density, potential energy, and inertia
respectively (Remark~\ref{rem:Tdecomp}); none is $T_{\mu\nu}$ directly,
and no conservation law $\nabla_\mu T^{\mu\nu}=0$ has been derived for
$\mathcal{T}[k]$.
(3)~\emph{Global geometry is only established in the spectral class.}
The closed-form geodesics and isometry of Section~\ref{sec:spectral} hold
for $\mathcal{K}_{\mathrm{graph}}$.  Extending these structures to general
kernel classes---Matérn, NTK, exponential families---requires resolving Q1
(computing $\mathcal{R}[k]$ in closed form beyond the spectral setting), which
remains open.
(4)~\emph{Conjecture~\ref{conj:triple} is a conjecture.}  No kernel
outside the spectral class has been proved a fixed point of the MaxCal
flow; computed examples are the only path to validation.

\paragraph{Scalability.}
Full eigendecomposition of $L$ costs $O(N^3)$, limiting the spectral
reduction to moderate-sized graphs ($N\lesssim 10^4$).  For larger
networks---drainage basins, communication topologies, large sensor
deployments---sparse iterative solvers (Lanczos, LOBPCG) deliver the
first $k\ll N$ eigenpairs in $O(kNm)$ time, where $m$ is the average
node degree.  Only the low-$\lambda$ end of the spectrum is needed for
phase-transition diagnostics ($\lambda_1$, $\mathcal{H}[h_t]$, $\Delta'$),
so this truncation preserves the framework's predictive content at
costs linear in $N$.  The Nystr\"{o}m approximation and graph
coarsening (METIS, Graclus) provide further scalability paths for
continental-scale networks~\cite{scholkopf2002}.

\subsection{Open problems}

\paragraph{Q1 (primary gateway).}
Computing $\mathcal{R}[k]$ in closed form for general kernel classes---
exponential families, Mat\'{e}rn, NTK---is the prerequisite for a full
geometric theory.  Proposition~\ref{prop:Q1resolved} resolves Q1 in the
spectral class; the remaining open component is the global metric on
$(\mathcal{K},g_k)$ beyond $\mathcal{K}_{\mathrm{graph}}$.

\begin{itemize}
\item \textbf{Q2} (Quantum kernel dynamics): Fubini--Study
      extension~\cite{bengtsson2006} to quantum kernel space.

\item \textbf{Q3} (Assembly index bound): $a(x)\geq c\|k_x\|_\mathcal{H}$
      ~\cite{sharma2023,walker2013}---an information-geometric positive-energy
      theorem connecting kernel complexity to selection history.

\item \textbf{Q4} (NTK convergence): whether gradient descent on finite-width
      networks flows toward the Hellinger attractor---a test of
      Conjecture~\ref{conj:triple}.

\item \textbf{Q5} (Adaptive sampling): numerical solution of the spectral
      kernel field equations~\eqref{eq:kernel-EFE} under operational boundary
      conditions, using $h(\lambda)$ as the $N$-vector state.
      Section~\ref{sec:numerics} provides a first concrete instantiation on
      $P_N$ with the Gaussian MI source~\eqref{eq:gmi-source}, verifying
      fixed-point convergence, geodesics, stability, and the early-warning
      diagnostics numerically.  The remaining open component is deploying
      this loop under real operational constraints (sensor scheduling,
      communication budgets).  An earlier
      closed-loop instantiation in sensor covariate space appears in
      adaptive marine sampling~\cite{das2015ijrr,das2013icra}: a GP kernel
      over environmental covariates was updated across successive AUV surveys,
      empirically tracing a discrete-time kernel trajectory converging toward
      an ecological attractor.

\item \textbf{Q6} (Coupling-aware early-warning diagnostics): When
      $\mathcal{T}$ introduces inter-modal coupling
      ($\partial\mathcal{T}_l/\partial h(\lambda_m)\neq 0$ for $m\neq l$),
      the scalar spectral entropy $\mathcal{H}[h_t]$ is a first-order
      marginal diagnostic that is blind to the off-diagonal structure of
      $\partial\mathcal{T}_l/\partial h(\lambda_m)$.  Cascading
      instabilities can be \emph{fully primed}---energy stored in modes
      $l\geq 2$ poised to flow into $l=1$ via coupled $\mathcal{T}$---while
      $\mathcal{H}[h_t]$ registers no anomaly.  Two complementary
      second-order diagnostics arise naturally from the existing framework:
      (i)~the \emph{Hessian gap}
      $\Delta(h_t)=\lambda_{\min}(-H(h_t))$,
      where $H_{lm}=-\delta_{lm}/h^*(\lambda_l)
      -\partial\mathcal{T}_l/\partial h(\lambda_m)$,
      which approaches zero at a fold bifurcation strictly before spectral
      mass concentrates; and (ii)~the \emph{per-mode coupling entropy}
      $\mathcal{S}_{\mathrm{coup}}(h_t)
       =N^{-1}\sum_l\bigl(-\sum_{m\neq l}p_{lm}\log p_{lm}\bigr)$
      with $p_{lm}=|\partial\mathcal{T}_l/\partial h(\lambda_m)|
      /\sum_{m'\neq l}|\partial\mathcal{T}_l/\partial h(\lambda_{m'})|$,
      whose decrease signals the Fiedler mode acquiring concentrated
      coupling partners before $\Delta$ closes.  The open question is
      whether the von Neumann entropy of the Fisher--Rao metric
      $\mathcal{S}_{\mathrm{vN}}=-\mathrm{Tr}(\hat{I}\log\hat{I})$,
      $\hat{I}=I(h)/\mathrm{Tr}(I(h))$, subsumes all three diagnostics:
      it reduces to $\mathcal{H}[h_t]$ in the mode-separable case and
      acquires off-diagonal contributions---encoding coupling---when
      $\mathcal{T}$ breaks mode separability.
\end{itemize}

%------------------------------------------------------------

\end{document}